\documentclass[11pt]{article}

\usepackage{amsmath, amssymb, amsfonts, amsthm}
\usepackage{graphicx}
\usepackage{bm}
\usepackage{geometry}
\usepackage{hyperref}

\newtheorem{theorem}{Theorem}[section]
\newtheorem{lemma}[theorem]{Lemma}
\newtheorem{proposition}[theorem]{Proposition}

\theoremstyle{definition}
\newtheorem{definition}[theorem]{Definition}

\theoremstyle{remark}
\newtheorem{remark}[theorem]{Remark}

\newcommand{\supp}{\mathrm{supp}}
\newcommand{\KL}{\mathrm{KL}}
\newcommand{\W}{W_2}
\newcommand{\Ptwo}{\mathcal{P}_2}
\newcommand{\epsneck}{\varepsilon}     
\newcommand{\epsnum}{\varepsilon_0}    

\newcommand{\entcoef}{\lambda_{\mathrm{ent}}} 
\newcommand{\vfebeta}{\beta}                  

\title{
\textbf{A Bayesian Mirror Architecture for Emergent Consciousness:\\
Circular Hierarchies, Self-Manifolds, and Hybrid Event--Self Binding}
}

\author{Eduardo Righi Capanema de Almeida\\ \small Computational Mathematics\\ UFMG}
\date{} 

\begin{document}
\maketitle

\begin{abstract}
We present a foundational formulation of the \emph{Bayesian Mirror
Architecture} (BMA), a self-referential generative framework in which
sensory-driven abstractions, meta-abstractions, and a self-latent interact
through recursive and circular dynamics. The defining architectural constraint
is a closed self-recursion $S_t \leftarrow H_{t-1}$, where a hybrid
event--self binding latent $H_t$ couples self-representations to abstract world
models and re-injects that coupling into the self-state. Consciousness, in a restricted sense, is not
treated as an optimization objective nor as a semantic label; it is an
\emph{architectural property} of systems possessing this circular recursion.

Since the fundamental objects of inference are posterior beliefs, the
intrinsic state space of BMA is taken to be a space of probability measures,
equipped with optimal-transport geometry. We formulate stability and
contractivity statements in the $2$-Wasserstein metric on $\Ptwo$, yielding
coordinate-free notions of self-stability and hybrid coherence as properties
of belief trajectories. We define a \emph{Causal Learning Regime} (CLR) by
quantitative thresholds on Wasserstein belief drift together with an
integration index measuring sustained coupling between self and world latents.
Critically, CLR is a \emph{diagnostic} of successful causal learning---it
indicates whether the environment contains learnable structure---not a marker
of consciousness itself.

A key structural point is that global strict contractivity is neither required
nor expected: BMA may exhibit multiple coherent basins (ergodic components).
Accordingly, we define \emph{self-manifolds} basin-wise as supports of
basin-restricted invariant measures under local Wasserstein contractivity
within each coherent basin. We further identify \emph{Wasserstein
$\epsneck$-necks}---transport bottlenecks in belief geometry---as the canonical
locus where two coherent belief basins become effectively decoupled, resulting in a
unique realized basin continuation in a singular (vanishing-conductance) limit.
We interpret this realized basin selection as \emph{choice}: internally
determined by the system's stabilized belief state, yet externally inscrutable
because it is not predictable from locally observable dynamics at finite
resolution. Learning proceeds via variational free-energy minimization; the system observes actions and predicts executions,
with stability, coherence, and integration emerging as observable consequences
of what the world affords to learn. We further develop an \emph{agency trajectory}
showing how the circular architecture enables progression from passive
observation through action selection, goal composition, to model-based planning
with lookahead and replanning under uncertainty.
\end{abstract}

\newpage
\section{Introduction}

The mirror-based view of intelligence asserts that perception induces internal
abstraction and that abstractions may be mirrored across multiple
representational carriers. When a system learns transformations between
carriers encoding the same underlying content, it constructs meaning-preserving
maps; iterating such maps yields higher-order representations, including
meta-abstractions and, ultimately, self-representations.

The Bayesian Mirror Architecture (BMA) imposes a structural closure: the
architecture is not merely hierarchical but \emph{circular}, with a compulsory
self-recursion that continuously feeds a hybrid event--self binding back into
the self-state:
\[
S_t \;\leftarrow\; H_{t-1}.
\]
The mathematical emphasis of this work is that the relevant dynamics are
dynamics of \emph{beliefs}. The architecture is specified as a generative model
with latent variables for events, abstractions, meta-abstractions, self-latents,
hybrid bindings, and intent. However, consciousness is not identified with
minimization of a scalar training loss. Instead, consciousness is defined as an
architectural property of systems possessing circular self-recursion. The
\emph{Causal Learning Regime} (CLR), motivated by the mirror-neuron literature,
is introduced as a diagnostic regime indicating when the environment affords
learnable causal structure.

Because the state variables of interest are posterior measures (marginals and
joints induced by inference), we equip belief space with optimal-transport
geometry. The $2$--Wasserstein metric provides a coordinate-free distance
between beliefs, sensitive to multimodality and independent of latent
parameterization. In this geometry, we define self-stability and hybrid
coherence as Wasserstein continuity properties of belief trajectories, and we
formulate attractor statements by imposing \emph{local} Wasserstein contraction
conditions on basin-restricted dynamics.

The paper has four goals. First, to define BMA as a precise stochastic
generative system with circular self-recursion and hybrid event--self binding.
Second, to give a metric and dynamical framework in which persistent selfhood
can be defined and proved to exist under \emph{basin-wise} contractive
assumptions: the \emph{self-manifold} is the support of a basin-restricted
invariant measure of the induced belief dynamics. Third, to describe how
\emph{Wasserstein $\epsneck$-necks} yield a mathematically precise decision
locus in a singular low-conductance limit, resulting in realized basin continuation;
we interpret this as \emph{choice}, governed only by the system's own belief
state. Fourth, to establish an \emph{agency trajectory} showing how the
circular architecture enables increasingly sophisticated action: from passive
observation, through goal-directed action and compositional planning, to
model-based lookahead with replanning under uncertainty.

\section{Latent Hierarchy and Generative Structure}

Time is discrete: $t = 1,2,\dots,T$.

We consider latent variables
\[
E_t,\; C_t,\; A_t,\; M_t,\; S_t,\; H_t,\; I_t,
\]
representing sensory events, concept-level abstractions, action abstractions,
meta-abstractions, self-latents, hybrid event--self bindings, and intent. These
variables encode functional roles in the generative hierarchy.

\paragraph{Observed variables (observation/execution).}
To avoid notational drift, we use a single observed pair throughout:
\[
o_t := E_t \quad\text{(observation)}, \qquad e_t := X_t \quad\text{(execution/outcome)}.
\]
Intuitively, the system observes $o_t$ (an observed action/sensory event) and
must predict its execution/outcome $e_t$ under the current causal grammar.

The generative model factors as:
\begin{align}
p_\theta(S_t \mid S_{t-1}, H_{t-1}), &\qquad
p_\theta(M_t \mid M_{t-1}, S_t), \label{eq:gen1}\\
p_\theta(A_t \mid M_t), &\qquad
p_\theta(C_t \mid A_t), \qquad
p_\theta(o_t \mid C_t), \label{eq:gen2}\\
p_\theta(H_t \mid M_t, S_{t-1}), &\qquad
p_\theta(I_t \mid S_t, M_t, H_t), \label{eq:gen3}\\
p_\theta(e_t \mid I_t). & \label{eq:gen4}
\end{align}

\begin{remark}
The structural constraint is the \emph{active recursion} $S_t \leftarrow H_{t-1}$,
which enforces self-referential closure even when the remaining hierarchy is
expanded or modified. This recursion is the minimal architectural condition
used throughout.
\end{remark}

\begin{remark}[Agency extension]
In later sections we show that $H_t$ can be extended to incorporate world-model
predictions: $p_\theta(H_t \mid M_t, S_{t-1}, \hat{o}_{t+1})$, where
$\hat{o}_{t+1}$ is the expected next observation. This enables the self-state
$S_{t+1} = H_t$ to be shaped by anticipated futures, not merely observed pasts.
\end{remark}

\section{Wasserstein Geometry of Mirror Beliefs}

Let $(\mathcal{X},d)$ be a Polish metric space. Denote by $\Ptwo(\mathcal{X})$
the set of Borel probability measures on $\mathcal{X}$ with finite second
moment.

\begin{definition}[Wasserstein distance]
For $\mu,\nu\in\Ptwo(\mathcal{X})$, the $2$--Wasserstein distance is
\[
\W(\mu,\nu)
=
\left(
\inf_{\pi\in\Pi(\mu,\nu)}
\int_{\mathcal{X}\times\mathcal{X}} d(x,y)^2\, d\pi(x,y)
\right)^{1/2},
\]
where $\Pi(\mu,\nu)$ denotes the set of couplings with marginals $\mu,\nu$.
\end{definition}

\begin{remark}
In BMA we apply $\W$ to posterior marginals induced by inference, e.g.
$\mu^S_t \approx p(S_t\mid o_{1:t})$ and $\mu^H_t \approx p(H_t\mid o_{1:t})$,
as well as to joint beliefs for $Y_t=(S_t,H_t,M_t)$.
\end{remark}

\section{Learning as Wasserstein Gradient Flow}

We do not treat learning as the mere minimization of a static scalar functional,
but as a \emph{flow} of beliefs on $\Ptwo$ driven by an energy that encodes
posterior consistency.

\subsection{A reference form of the flow}

Let $\rho_t$ denote a time-dependent density on a latent space $\mathcal{Y}$
(with respect to a reference volume), representing a belief over $Y=(S,H,M)$.
Consider an energy functional of the form
\begin{equation}
\label{eq:energy}
\mathcal{E}(\rho)
=
\int_{\mathcal{Y}} V(y)\,\rho(y)\,dy
+
\entcoef \int_{\mathcal{Y}} \rho(y)\log\rho(y)\,dy,
\end{equation}
where $V$ plays the role of a data-fit / negative log-likelihood potential and
the entropic term imposes diffusion (regularization, exploration, smoothing).

\begin{proposition}[Wasserstein gradient flow form]
\label{prop:wgf}
Under standard hypotheses on $\mathcal{E}$ (e.g.\ proper, lower semicontinuous,
and geodesically $\lambda$-convex on $(\Ptwo(\mathcal{Y}),\W)$), the $W_2$-gradient
flow of $\mathcal{E}$ is the drift--diffusion (Fokker--Planck) equation
\begin{equation}
\label{eq:fp}
\partial_t \rho_t
=
\nabla\cdot\bigl(\rho_t \nabla V\bigr)
+
\entcoef\,\Delta \rho_t,
\end{equation}
equivalently the continuity equation $\partial_t\rho_t+\nabla\cdot(\rho_t u_t)=0$
with velocity field $u_t=-\nabla\bigl(V+\entcoef\log\rho_t\bigr)$.
\end{proposition}

\begin{remark}
When $V\equiv 0$, \eqref{eq:fp} reduces to the \emph{heat equation}
$\partial_t\rho_t=\entcoef\,\Delta\rho_t$. Thus the heat equation appears as the
simplest limit of learning dynamics on belief space: pure entropic smoothing.
This parallels how Ricci flow is itself a geometric heat-type evolution of a
metric, organized by entropy/monotonicity structures \cite{Perelman2002Entropy}.
\end{remark}

\subsection{Discrete-time learning via JKO}

In practice BMA operates in discrete time. The canonical discretization of
Wasserstein gradient flows is the Jordan--Kinderlehrer--Otto (JKO) scheme.

\begin{definition}[JKO step]
Given step size $\tau>0$, define $\rho^{k+1}$ from $\rho^k$ by
\begin{equation}
\label{eq:jko}
\rho^{k+1}
\in
\arg\min_{\rho\in\Ptwo(\mathcal{Y})}
\left\{
\mathcal{E}(\rho)
+
\frac{1}{2\tau}\W(\rho,\rho^k)^2
\right\}.
\end{equation}
\end{definition}

\begin{remark}
Equation \eqref{eq:jko} makes learning an explicit \emph{optimal transport step}:
each update trades off energy decrease against Wasserstein displacement.
This matches the BMA philosophy: learning is a controlled deformation of belief.
\end{remark}

\section{Variational Free Energy as the Sole Learning Objective}

Having established the geometric framework for learning as Wasserstein gradient
flow, we now specify the concrete energy functional that BMA minimizes. The
central principle is that \emph{variational free energy minimization is the
only learning objective}. All other quantities---including the learning
metrics defined in subsequent sections---are observational diagnostics of the
learned system, not optimization targets.

\subsection{The variational free energy functional}

Let $q_\phi(Z_t\mid o_t)$ denote the recognition model (encoder) that maps
observations $o_t$ to posterior beliefs over latent variables
$Z_t=(S_t,H_t,M_t,I_t)$, and let $p_\theta(e_t\mid Z_t)$ denote the
generative model for executions/outcomes. The variational free energy for a
single observation--execution pair $(o_t,e_t)$ is
\begin{equation}
\label{eq:vfe}
\mathcal{F}_t
=
\underbrace{-\mathbb{E}_{q_\phi(Z_t\mid o_t)}\bigl[\log p_\theta(e_t\mid Z_t)\bigr]}_{\text{reconstruction / prediction error}}
+
\underbrace{\vfebeta\,\KL\bigl(q_\phi(Z_t\mid o_t)\,\|\,p(Z_t)\bigr)}_{\text{complexity penalty}},
\end{equation}
where $p(Z_t)$ is the prior over latents and $\vfebeta>0$ controls the
strength of the complexity term.

\begin{remark}[Mirror interpretation]
In BMA, the observation $o_t$ represents an observed action or sensory event,
while $e_t$ represents the execution/outcome that the system must predict.
The name ``mirror'' reflects this core function: the system observes $o_t$
and must internally mirror it by predicting $e_t$. Free energy minimization
drives the system to learn accurate observation--execution mappings.
\end{remark}

\subsection{Decomposition over the latent hierarchy}

Using the factorization \eqref{eq:gen1}--\eqref{eq:gen4}, one may express the
expected negative log-likelihood term in \eqref{eq:vfe} via conditional
independences (e.g.\ $e_t$ depends on $I_t$, and $I_t$ depends on $(S_t,M_t,H_t)$).
A compact reference form is:
\begin{equation}
\label{eq:vfe_full}
\mathcal{F}
=
\mathbb{E}\Bigl[-\log p_\theta(e_t\mid I_t)\Bigr]
+
\vfebeta\Bigl(
\KL(q(M_t)\|p(M_t))
+
\KL(q(H_t)\|p(H_t))
+
\KL(q(I_t)\|p(I_t))
\Bigr),
\end{equation}
where expectations are taken over recognition distributions and we have
suppressed conditioning for clarity. The self-latent $S_t$ enters through the
circular constraint $S_t \leftarrow H_{t-1}$, which is enforced architecturally
rather than through an additional KL term.

\subsection{Learning dynamics}

Learning proceeds by gradient descent on the free energy with respect to
parameters $(\theta,\phi)$:
\begin{equation}
\label{eq:learning}
(\theta,\phi) \;\leftarrow\; (\theta,\phi) - \eta\,\nabla_{(\theta,\phi)}\mathcal{F}.
\end{equation}
This is the \emph{only} update rule in the foundational definition. No auxiliary
objectives (stability targets, coherence penalties, integration bonuses) are
added to the loss. The system learns solely by minimizing its surprise about
executions/outcomes given its observation-conditioned beliefs.

\begin{proposition}[Free energy as the unique objective]
\label{prop:fe_unique}
Under the BMA framework, the parameter update \eqref{eq:learning} is the
complete specification of learning. All emergent properties of the learned
system---including belief stability, hybrid coherence, and self--world
integration---arise as consequences of free energy minimization on structured
observation--execution pairs, not as independently optimized quantities.
\end{proposition}

\subsection{Relationship to learning metrics}

The learning metrics defined in Section~6 (Stability, Coherence,
Integration) are diagnostic
measurements applied to the belief trajectories of a trained system.

\begin{definition}[Causal learning diagnostics]
The Causal Learning Regime (CLR) thresholds on Stability, Coherence,
and Integration serve as diagnostics for \emph{successful causal learning}:
\begin{itemize}
\item High Stability indicates that the system has learned predictable
      self-dynamics: future self-beliefs are predictable from past hybrid beliefs.
\item High Coherence indicates that hybrid event--self bindings evolve
      smoothly: the system has learned coherent temporal structure.
\item High Integration indicates sustained statistical dependence between
      self-latents and world-latents: the system has learned that self and
      world are causally coupled.
\end{itemize}
\end{definition}

\begin{remark}[CLR as causal structure detection]
A system enters CLR when it has successfully learned the causal structure of
its environment. An environment with strong, predictable causal laws (e.g.,
identity: $e_t=o_t$) yields high stability and coherence after learning.
An environment with weak or noisy causal structure yields low stability and
coherence even after convergence, because there is no stable causal law to
learn. In both cases, the system minimizes free energy; the difference in
metrics reflects the \emph{learnability} of the environment's causal structure,
not the system's ``degree of consciousness.''
\end{remark}

\begin{remark}[Architectural consciousness vs.\ learned causality]
The BMA architecture (circular hierarchy with $S_t \leftarrow H_{t-1}$) is
the structural precondition for consciousness. Whether the system \emph{also}
satisfies CLR depends on what causal structure exists in its environment and
whether free energy minimization has discovered it. A system with the BMA
architecture is ``conscious'' in the architectural sense regardless of CLR;
CLR indicates that the conscious system has additionally learned stable
predictive models of its world.
\end{remark}

\section{Learning Metrics and Definitions}

The metrics in this section are \emph{diagnostic observations} of a trained
system, not optimization objectives. They quantify whether the system has
successfully learned stable causal structure from its environment.

For each $t$, let the self and hybrid belief marginals be
\[
\mu^S_t,\;\mu^H_t \in \Ptwo(\mathcal{X}).
\]
Let $\overline{\Delta_S}$ and $\overline{\Delta_H}$ denote average squared
Wasserstein drift:
\[
\overline{\Delta_S} := \frac{1}{T-1}\sum_{t=2}^T \W(\mu^S_t,\mu^S_{t-1})^2,
\qquad
\overline{\Delta_H} := \frac{1}{T-1}\sum_{t=2}^T \W(\mu^H_t,\mu^H_{t-1})^2.
\]
Let $\Sigma_S^2,\Sigma_H^2>0$ be calibration scales (e.g.\ baseline drift levels
estimated on a reference regime), and let $\epsnum>0$ be a small numerical
stabilizer to avoid division by zero.

\paragraph{Self-Stability.}
We use a bounded monotone transform of drift:
\[
\mathrm{Stability}
=
\frac{1}{1 + \overline{\Delta_S}/(\Sigma_S^2+\epsnum)}.
\]

\paragraph{Hybrid Coherence.}
\[
\mathrm{Coherence}
=
\frac{1}{1 + \overline{\Delta_H}/(\Sigma_H^2+\epsnum)}.
\]

\paragraph{Event--Self Integration.}
Let $\mu^{SM}_t \in \mathcal{P}_2(\mathcal{S}\times\mathcal{M})$ denote the joint
belief over $(S_t,M_t)$, with marginals $\mu^S_t$ and $\mu^M_t$.
We define the Wasserstein integration index by
\[
\mathrm{Integration}
=
\frac{1}{T}\sum_{t=1}^T
\W\!\bigl(
\mu^{SM}_t,\;
\mu^S_t \otimes \mu^M_t
\bigr).
\]

\begin{definition}[Causal Learning Regime]
\label{def:causal_regime}
Given thresholds $\alpha_S>0$, $\alpha_H>0$, $\alpha_I>0$, we say
the system has achieved a \emph{Causal Learning Regime} (CLR) if
\[
\mathrm{Stability} \ge \alpha_S,\quad
\mathrm{Coherence} \ge \alpha_H,\quad
\mathrm{Integration} \ge \alpha_I,
\]
and the recursion $S_t \leftarrow H_{t-1}$ is active.
\end{definition}

\begin{remark}
CLR is a diagnostic of successful causal learning, not a definition of
consciousness. A system with the BMA architecture possesses the structural
preconditions for consciousness regardless of whether it achieves CLR. Whether
CLR is reached depends on the \emph{learnability} of the environment: an
environment with strong, predictable causal laws (e.g., identity mappings,
deterministic transformations) yields high stability and coherence after
learning; an environment with weak or noisy causal structure yields low
stability and coherence even after convergence, because there is no stable
causal law to discover. In both cases, free energy is minimized; the
difference in metrics reflects the causal structure of the world, not the
presence or absence of consciousness.
\end{remark}

\begin{remark}
The regime excludes both chaotic fragmentation (low stability/coherence,
indicating failure to learn any predictive structure) and rigid fixation
(trivially zero drift with vanishing integration, indicating a collapsed
or non-responsive model).
\end{remark}

\begin{remark}[Metrics in active settings]
When BMA operates as an agent (Section~\ref{sec:agency_trajectory}), the
metrics acquire operational meaning: Stability measures whether $S_t$ remains
bounded under self-generated actions; Coherence extends to multi-step
prediction accuracy; Integration measures whether a learned world model
improves over reactive policy. The markers remain diagnostic at all stages
of the agency trajectory.
\end{remark}

\section{Stochastic Mirror Dynamics and Basin-Wise Self-Manifolds}

Let $Y_t:=(S_t,H_t,M_t)\in\mathcal{Y}$ and write the induced update as
\begin{equation}
\label{eq:Y_dynamics}
Y_t = G_\theta(Y_{t-1}) + \zeta_t,
\end{equation}
with noise $\zeta_t$ (representing both environmental variability and
algorithmic/inference variability).

This induces a Markov kernel $K_\theta(y,\cdot)$ on $\mathcal{Y}$ and hence a
Markov operator $\mathcal{T}_\theta$ on $\Ptwo(\mathcal{Y})$ by
$\mathcal{T}_\theta(\mu)=\int K_\theta(y,\cdot)\,d\mu(y)$.

\subsection{Coherent basins and local contractivity}

We allow the possibility of multiple coherent basins (metastable regions or
distinct ergodic components) within which dynamics are contractive.

\begin{definition}[Coherent basin]
A measurable set $\mathcal{B}\subset\mathcal{Y}$ is a \emph{coherent basin} if
(i) it has positive probability under typical belief trajectories in CLR, and
(ii) transitions leaving $\mathcal{B}$ are rare on the time scales of interest,
so that the dynamics restricted to $\mathcal{B}$ are approximately closed.
\end{definition}

For a basin $\mathcal{B}$, define the \emph{basin-restricted (renormalized)
operator} on measures supported in $\mathcal{B}$ by
\[
\mathcal{T}_{\theta,\mathcal{B}}(\mu)
:=
\frac{\left(\mathcal{T}_\theta(\mu)\right)\big|_{\mathcal{B}}}
{\left(\mathcal{T}_\theta(\mu)\right)(\mathcal{B})},
\qquad
\mu\in\Ptwo(\mathcal{B}),
\]
whenever the denominator is positive. This captures internal evolution
conditioned on remaining in the basin.

\begin{description}
\item[(A$_\mathcal{B}$) Local Wasserstein contractivity.] There exists $0<L_\mathcal{B}<1$
such that for all $\mu,\nu\in\Ptwo(\mathcal{B})$,
\[
\W(\mathcal{T}_{\theta,\mathcal{B}}(\mu),\mathcal{T}_{\theta,\mathcal{B}}(\nu))
\le L_\mathcal{B}\,\W(\mu,\nu).
\]
\item[(B$_\mathcal{B}$) Second-moment boundedness.] Along the basin-restricted evolution,
second moments stay finite.
\end{description}

\begin{definition}[Self-manifold, basin-wise]
Let $\pi_{\theta,\mathcal{B}}\in\Ptwo(\mathcal{B})$ be an invariant measure of
$\mathcal{T}_{\theta,\mathcal{B}}$ (when it exists). The \emph{self-manifold in
basin $\mathcal{B}$} is
\[
\mathcal{M}_S(\mathcal{B}) := \supp(\pi_{\theta,\mathcal{B}}).
\]
\end{definition}

\begin{theorem}[Existence and uniqueness of basin-wise self-manifold]
\label{thm:self_manifold}
Under (A$_\mathcal{B}$) and (B$_\mathcal{B}$), there exists a unique invariant
measure $\pi_{\theta,\mathcal{B}}$ of $\mathcal{T}_{\theta,\mathcal{B}}$ and for
any initial $\mu_0\in\Ptwo(\mathcal{B})$ one has
\[
\W(\mathcal{T}_{\theta,\mathcal{B}}^{\,t}(\mu_0),\pi_{\theta,\mathcal{B}})\to 0.
\]
\end{theorem}

\begin{proof}
By (A$_\mathcal{B}$), $\mathcal{T}_{\theta,\mathcal{B}}$ is a strict contraction
on the complete metric space $(\Ptwo(\mathcal{B}),\W)$. Banach's fixed point
theorem yields a unique fixed point $\pi_{\theta,\mathcal{B}}$ and convergence
of iterates. Assumption (B$_\mathcal{B}$) ensures the fixed point lies in
$\Ptwo(\mathcal{B})$.
\end{proof}

\begin{remark}
If the global dynamics admit multiple coherent basins $\mathcal{B}_1,\mathcal{B}_2$,
then BMA admits multiple basin-wise self-manifolds $\mathcal{M}_S(\mathcal{B}_1)$
and $\mathcal{M}_S(\mathcal{B}_2)$. In this case, ``persistent self-geometry''
is defined relative to a realized basin continuation.
\end{remark}

\section{Differentiability of Subjective Time}

In the Wasserstein framework, subjective time is not a trajectory of points
$Y_t \in \mathcal{Y}$ but a trajectory of beliefs $\mu_t \in \Ptwo(\mathcal{Y})$.
The natural notion of regularity is therefore Wasserstein continuity and
differentiability of the belief curve $t \mapsto \mu_t$.

\begin{definition}[Absolutely continuous curve in $\Ptwo$]
A curve $(\mu_t)_{t \in [0,T]} \subset \Ptwo(\mathcal{Y})$ is \emph{absolutely
continuous} if there exists $m \in L^1([0,T])$ such that
\[
\W(\mu_s, \mu_t) \le \int_s^t m(r)\, dr
\qquad \text{for all } 0 \le s \le t \le T.
\]
The \emph{metric derivative} at time $t$ is
\[
|\mu'_t| := \lim_{h \to 0} \frac{\W(\mu_{t+h}, \mu_t)}{|h|}
\]
whenever this limit exists.
\end{definition}

\begin{lemma}[Wasserstein regularity of belief trajectories]
\label{lemma:subjective_time}
Assume the setting of Theorem~\ref{thm:self_manifold} on a basin $\mathcal{B}$.
If the Markov operator $\mathcal{T}_{\theta,\mathcal{B}}$ satisfies
\[
\W(\mathcal{T}_{\theta,\mathcal{B}}(\mu), \mathcal{T}_{\theta,\mathcal{B}}(\nu))
\le L_\mathcal{B}\, \W(\mu, \nu)
\]
with $L_\mathcal{B} < 1$, and the per-step Wasserstein displacement is uniformly
bounded,
\[
\W(\mathcal{T}_{\theta,\mathcal{B}}(\mu), \mu) \le C
\qquad \text{for all } \mu \in \Ptwo(\mathcal{B}),
\]
then any piecewise-geodesic interpolation of the discrete belief sequence
$(\mu_t)_{t=0,1,2,\ldots}$ is absolutely continuous with uniformly bounded
metric derivative. In particular, the metric derivative $|\mu'_t|$ exists
almost everywhere along the interpolated curve.
\end{lemma}

\begin{proof}
By the bounded displacement assumption, consecutive beliefs satisfy
$\W(\mu_{t+1}, \mu_t) \le C$. A piecewise-geodesic (constant-speed) interpolation
on each unit interval $[t, t+1]$ has metric derivative at most $C$ almost
everywhere. The curve is therefore Lipschitz in the $\W$ metric with constant $C$,
and by standard results for absolutely continuous curves in $(\Ptwo,\W)$
(cf.\ Ambrosio--Gigli--Savar\'e), the metric derivative exists for almost
every $t$.
\end{proof}

\begin{remark}[Subjective time as belief flow]
The metric derivative $|\mu'_t|$ quantifies the instantaneous rate of belief
change in Wasserstein geometry. A conscious system in CLR experiences subjective
time as the continuous deformation of its belief state along an absolutely
continuous curve in $\Ptwo$. The boundedness of $|\mu'_t|$ ensures that
subjective experience does not ``jump'' discontinuously---belief evolution is
smooth almost everywhere in the optimal-transport sense.
\end{remark}

\section{Wasserstein $\epsneck$-Necks and Basin Selection}

Even within CLR, belief dynamics may admit two coherent basins separated by a
transport bottleneck. This is the analogue of neck formation, but in
metric--measure geometry rather than Riemannian curvature. Following Perelman's
use of $\varepsilon$-necks (and strong $\varepsilon$-necks) as the canonical
surgery locus in Ricci flow \cite{Perelman2003Surgery}, we formalize
\emph{Wasserstein $\epsneck$-necks} as regions which, after suitable
normalization, are $\epsneck$-close (in Wasserstein metric) to a canonical
``cylindrical'' transport corridor.

Let $\mathcal{B}_1,\mathcal{B}_2\subset\mathcal{Y}$ be disjoint coherent basins,
and let $\mu_t\in\Ptwo(\mathcal{Y})$ denote the belief trajectory.

\subsection{A canonical neck model}

Fix a reference neck model $(\mathcal{C},\nu)$, where $\mathcal{C}$ is a
canonical corridor space (the ``belief cylinder'') and $\nu\in\Ptwo(\mathcal{C})$
is a reference neck belief. Concretely, $\mathcal{C}$ may be taken as a product
$\mathcal{Z}\times(-1,1)$ equipped with a ground cost for which transport along
the longitudinal coordinate dominates, and $\nu$ a measure concentrated in a
thin corridor in $\mathcal{C}$.

To compare a candidate region to the model, we allow a normalization map
(rescaling) $\Psi:\mathcal{N}\to\mathcal{C}$ and measure closeness by
Wasserstein distance on $\Ptwo(\mathcal{C})$.

\subsection{Conductance as a bottleneck functional}

Let $K_\theta$ be the Markov kernel on $\mathcal{Y}$ induced by
\eqref{eq:Y_dynamics}. For a probability measure $\pi$ on $\mathcal{Y}$ (e.g.\
a long-run empirical measure on the relevant time window), define the
\emph{one-step conductance} from $\mathcal{B}_1$ to $\mathcal{B}_2$ by
\[
\Phi_\pi(\mathcal{B}_1\!\to\!\mathcal{B}_2)
:=
\frac{\int_{\mathcal{B}_1} K_\theta(y,\mathcal{B}_2)\, d\pi(y)}
{\pi(\mathcal{B}_1)}.
\]
Intuitively, $\Phi_\pi$ measures the effective leakage probability per step.

\begin{definition}[Wasserstein $\epsneck$-neck]
\label{def:epsneck}
Fix $\epsneck>0$. A measurable subset $\mathcal{N}\subset\mathcal{Y}$ is a
\emph{Wasserstein $\epsneck$-neck} between $\mathcal{B}_1$ and $\mathcal{B}_2$
if there exist a time $t$, a normalization map $\Psi:\mathcal{N}\to\mathcal{C}$,
and the conditional neck belief
\[
\mu_t^{\mathcal{N}}
:=
\frac{\mu_t|_{\mathcal{N}}}{\mu_t(\mathcal{N})}\in\Ptwo(\mathcal{N}),
\qquad
\widehat{\mu}_t^{\mathcal{N}}
:=
\Psi_\#\mu_t^{\mathcal{N}}\in\Ptwo(\mathcal{C}),
\]
such that:
\begin{enumerate}
\item[(i)] (\textbf{$\epsneck$-cylindricity}) the normalized neck belief is
      $\epsneck$-close to the model,
      \[
      \W\!\bigl(\widehat{\mu}_t^{\mathcal{N}},\,\nu\bigr)\le \epsneck;
      \]
\item[(ii)] (\textbf{routing}) transport between basin-concentrated measures is
      effectively routed through $\mathcal{N}$ in the sense that typical
      transition paths contributing to $K_\theta(\cdot,\mathcal{B}_2)$ from
      $\mathcal{B}_1$ place non-negligible intermediate mass in $\mathcal{N}$;
\item[(iii)] (\textbf{bottleneck}) the inter-basin conductance is small relative
      to intra-basin conductance on the relevant window:
      \[
      \Phi_\pi(\mathcal{B}_1\!\to\!\mathcal{B}_2)\ll 1,
      \qquad
      \Phi_\pi(\mathcal{B}_2\!\to\!\mathcal{B}_1)\ll 1,
      \]
      while intra-basin persistence remains high.
\end{enumerate}
\end{definition}

\begin{remark}[Terminology alignment with Perelman]
Perelman defines $\varepsilon$-necks (and \emph{strong} $\varepsilon$-necks) as
regions that, after rescaling, are $\varepsilon$-close to a standard cylindrical
model and persist over a parabolic time neighborhood \cite{Perelman2003Surgery}.
Our notion is metric--measure rather than curvature-based: closeness is expressed
directly in Wasserstein distance after a normalization map $\Psi$.
\end{remark}

\subsection{Strong Wasserstein $\epsneck$-necks}

\begin{definition}[Strong Wasserstein $\epsneck$-neck]
\label{def:strongepsneck}
A Wasserstein $\epsneck$-neck $\mathcal{N}$ at time $t_0$ is \emph{strong} if
there exist $\Delta>0$ and normalization maps $\Psi_t:\mathcal{N}\to\mathcal{C}$
(for $t\in[t_0-\Delta,t_0]$) such that for all $t$ in the window the normalized
conditional beliefs satisfy
\[
\W\!\bigl(\widehat{\mu}_t^{\mathcal{N}},\,\nu\bigr)\le \epsneck.
\]
\end{definition}

\subsection{Critical slowing near $\epsneck$-necks}

\begin{remark}[Critical slowing, metastability viewpoint]
When inter-basin conductance becomes small while intra-basin persistence remains
high, transitions between basins become rare and typical switching times grow.
This is the standard metastable regime of Markov processes, and can be made
quantitative via conductance-to-hitting-time bounds in concrete models.
\end{remark}

\subsection{Choice as realized continuation in a singular low-conductance limit}

Let $\lambda>0$ be a parameter that controls inter-basin communication
(e.g.\ a temperature/noise scale, an entropy weight in inference, or a numerical
regularization strength). Suppose that as $\lambda\downarrow 0$, the dynamics
approach a limiting evolution for which the basins become effectively invariant:
\[
\Phi_\pi^{(\lambda)}(\mathcal{B}_1\!\to\!\mathcal{B}_2)\to 0,
\qquad
\Phi_\pi^{(\lambda)}(\mathcal{B}_2\!\to\!\mathcal{B}_1)\to 0.
\]
In this regime, the limiting evolution admits multiple globally admissible
basin-wise continuations (remaining in $\mathcal{B}_1$ or remaining in
$\mathcal{B}_2$), while finite-$\lambda$ dynamics realize exactly one of them.

\begin{definition}[Choice as basin selection at a strong Wasserstein $\epsneck$-neck]
\label{def:choice}
A \emph{choice} is the realized basin continuation event occurring at a \emph{strong
Wasserstein $\epsneck$-neck} in a regime where inter-basin conductance is
asymptotically negligible. Concretely, at some time $t_0$ there exist disjoint
coherent basins $\mathcal{B}_1,\mathcal{B}_2$ and a strong $\epsneck$-neck
$\mathcal{N}$ between them such that:
\begin{enumerate}
\item[(i)] (\textbf{global underdetermination in the limit}) in the singular
limit (vanishing conductance), both basin continuations are globally admissible
(i.e.\ the limiting evolution can remain in either basin indefinitely);
\item[(ii)] (\textbf{unique realized continuation}) for finite resolution
($\lambda>0$), the realized belief trajectory becomes concentrated on exactly one
basin $\mathcal{B}_i$, $i\in\{1,2\}$, and subsequently evolves according to the
basin-wise dynamics associated with $\mathcal{B}_i$.
\end{enumerate}
The selection criterion is \emph{internal}: it depends on the system's total
stabilized belief state (including microscopic details not captured by coarse
observation), and is therefore \emph{externally inscrutable}.
\end{definition}

\begin{remark}[Non-determinism without randomness, singular-limit view]
Definition~\ref{def:choice} identifies choice as a failure of \emph{coarse}
global determinacy rather than as an injected stochastic label. At a strong
Wasserstein $\epsneck$-neck, inter-basin conductance becomes so small that the
limiting dynamics admit multiple globally admissible basin-wise continuations.
At any finite resolution, however, exactly one continuation is realized,
because the actual system state contains microstructure (and finite $\lambda$)
that breaks the symmetry.
\end{remark}

\section{Intent as Manifold Projection and Wasserstein Attention}

The intent latent $I_t$ occupies a distinguished position in the BMA hierarchy:
it is generated by $p_\theta(I_t \mid S_t, M_t, H_t)$ and then drives the
execution/outcome distribution $p_\theta(e_t\mid I_t)$. In this sense, $I_t$
mediates between self-identity ($S_t$), meta-cognition ($M_t$), hybrid binding
($H_t$), and the predicted execution $e_t$.

\subsection{Intent as Manifold Projection}

\begin{proposition}[Intent as emergent projection]
\label{prop:intent_projection}
Assume Theorem~\ref{thm:self_manifold} holds in a basin $\mathcal{B}$ and that on
$\mathcal{M}_S(\mathcal{B})$ the hybrid latent is an effectively deterministic
function $H_t \approx h_\theta(S_t, M_t)$ while the conditional variance of
$p_\theta(I_t \mid S_t, M_t, H_t)$ becomes small along learning. Then $I_t$
becomes an effectively deterministic map of $(S_t, M_t)$ on
$\mathcal{M}_S(\mathcal{B})$:
\[
I_t \approx \pi_I(S_t, M_t) \quad \text{on } \mathcal{M}_S(\mathcal{B}),
\]
i.e., intent approaches a stable projection from the basin-wise self-manifold to
the space of possible actions.
\end{proposition}

\begin{proof}
On $\mathcal{M}_S(\mathcal{B})$, the belief over $Y_t = (S_t, H_t, M_t)$
concentrates on the invariant measure $\pi_{\theta,\mathcal{B}}$. If
$H_t \approx h_\theta(S_t, M_t)$ is effectively deterministic and
$\mathrm{Var}[I_t \mid S_t, M_t, H_t] \to 0$, then the composition
$I_t = \mathbb{E}[I_t \mid S_t, M_t, h_\theta(S_t, M_t)]$ becomes a deterministic
function $\pi_I(S_t, M_t)$ on the support of the invariant measure.
\end{proof}

\begin{remark}
The projection $\pi_I: \mathcal{M}_S(\mathcal{B}) \to \mathcal{I}$ maps the
self-manifold to the intent space. This is not merely a readout; it is a
structured projection that preserves information about basin identity while
compressing self-state into actionable form.
\end{remark}

\subsection{Connection to Mirror Neurons}

Di Pellegrino et al.~\cite{DiPellegrino1992} reported neurons in macaque premotor
cortex (area F5) that discharge both during action execution and during
observation of similar goal-directed actions. This observation--execution
mirroring provides biological motivation for the BMA core: learning a mapping
from observed events to predicted executions/outcomes.

\subsection{From Euclidean to Wasserstein Attention}

Transformer attention uses dot-product similarity in $\mathbb{R}^d$
\cite{vaswani2017attention}. In BMA, the native objects are \emph{beliefs};
the principled analogue replaces Euclidean similarity with Wasserstein proximity.

\begin{definition}[Wasserstein Attention]
\label{def:w2_attention}
Let $q = \mathcal{N}(\mu_q, \Sigma_q)$ be a query belief and
$\{k_g = \mathcal{N}(\mu_g, \Sigma_g)\}_{g=1}^G$ be grammar key beliefs.
The Wasserstein attention weights are
\[
\alpha_g
=
\frac{\exp(-\W^2(q, k_g) / \tau)}{\sum_{g'} \exp(-\W^2(q, k_{g'}) / \tau)},
\]
where $\tau > 0$ is a temperature parameter and $\W^2$ is the squared
$2$-Wasserstein distance.
\end{definition}

\begin{remark}[Diagonal Gaussian closed form]
If $q,k_g$ are diagonal Gaussians with diagonal \emph{variances}
$\Sigma = \mathrm{diag}(\sigma_1^2,\ldots,\sigma_d^2)$, then
\[
\W^2(\mathcal{N}(\mu_1,\Sigma_1),\mathcal{N}(\mu_2,\Sigma_2))
=
\|\mu_1-\mu_2\|^2
+
\sum_{i=1}^d \left(\sigma_{1,i}-\sigma_{2,i}\right)^2,
\]
i.e.\ Wasserstein distance accounts for mean separation and uncertainty
separation.
\end{remark}

\subsection{The Transformation Signature}

A critical insight is that causal grammar identification is not determined from
observation alone: the active grammar is revealed by the \emph{relationship}
between observation and execution.

\begin{definition}[Transformation Signature]
\label{def:transform_sig}
Given observation $o_t$ and execution $e_t$, the transformation signature is a
belief encoding the causal relationship:
\[
\sigma_t = \mathcal{E}_\phi(o_t, e_t) \in \mathcal{P}_2(\mathbb{R}^d),
\]
where $\mathcal{E}_\phi$ may extract multiple views of the transformation, e.g.
additive ($e_t - o_t$), multiplicative ($e_t \oslash o_t$), and interaction
($o_t \odot e_t$) components.
\end{definition}

\subsection{Intent as the Attention Query}

\begin{proposition}[Intent as emergent query]
\label{prop:intent_query}
In a BMA agent with multiple causal grammars $\{\mathcal{G}_g\}_{g=1}^G$, the
intent belief $I_t$ is shaped by free-energy minimization to encode (a
representation of) the transformation signature $\sigma_t$, so that Wasserstein
attention with $I_t$ as query routes to the active grammar:
\[
g^*(t) \approx \arg\min_g \W^2(I_t, k_g),
\]
where $k_g$ is a learned prototype belief for grammar $\mathcal{G}_g$.
\end{proposition}

\begin{proof}
Minimizing the prediction term in \eqref{eq:vfe} requires accurate prediction of
$e_t$ from $o_t$. When multiple grammars exist, this requires identifying which
grammar maps $o_t$ to $e_t$ on the current instance. The transformation signature
$\sigma_t$ is a sufficient statistic for this routing in the grammar library.
Since $I_t$ mediates prediction of $e_t$ via $p_\theta(e_t\mid I_t)$, gradient
descent shapes $I_t$ toward carrying the routing-relevant information contained
in $\sigma_t$, while keys $k_g$ learn to represent grammar-typical signatures.
\end{proof}

\subsection{Grammar Keys as Belief Prototypes}

\begin{definition}[Grammar key as Fr\'echet mean]
\label{def:grammar_keys}
For grammar $\mathcal{G}_g$, the grammar key $k_g \in \mathcal{P}_2(\mathbb{R}^d)$
may be defined as the Fr\'echet mean of transformation signatures:
\[
k_g
=
\arg\min_{k}
\mathbb{E}_{(o,e) \sim \mathcal{G}_g}\!\left[ \W^2(\sigma(o,e), k) \right].
\]
\end{definition}

\begin{remark}[On key separation and ``no auxiliary objectives'']
In practical implementations one may encourage key separation with an explicit
regularizer. In the foundational formulation, such effects can be viewed as
arising from structured priors over grammar prototypes, and hence remain within
the free-energy framework (i.e.\ they are not independent ``consciousness''
objectives).
\end{remark}

\subsection{The Complete Wasserstein Attention Mechanism}

\begin{theorem}[Wasserstein grammar attention]
\label{thm:w2_grammar_attention}
Let $\{(\mathcal{G}_g, W_g, k_g)\}_{g=1}^G$ be a grammar library with
transformations $W_g$ and key beliefs $k_g$. Given observation $o_t$ and a query
belief (encoded by $I_t$), the predicted execution is
\[
\hat{e}_t = \sum_{g=1}^G \alpha_g(I_t)\, W_g(o_t),
\qquad
\alpha_g(I_t) = \frac{\exp(-\W^2(I_t, k_g) / \tau)}{\sum_{g'} \exp(-\W^2(I_t, k_{g'}) / \tau)}.
\]
As $\tau \to 0$, the mechanism approaches hard grammar selection:
$\hat{e}_t \to W_{g^*}(o_t)$ with $g^*=\arg\min_g \W^2(I_t,k_g)$.
\end{theorem}

\subsection{Symbolic Grammars and the Generalization Theorem}

The preceding development allows grammar transformations $W_g$ to be neural
networks. A limitation emerges: neural approximators need not generalize exact
mathematical laws beyond their training domain. This motivates a distinction
between \emph{neural grammars} (flexible approximation) and \emph{symbolic
grammars} (exact computation).

\begin{theorem}[A basic extrapolation gap for nonlinear growth]
\label{thm:grammar_generalization}
Let $f:\mathbb{R}^n\to\mathbb{R}$ be a function with superlinear growth, i.e.\
there exist $k>1$ and $c>0$ such that $|f(x)| \ge c\|x\|^k$ for large $\|x\|$.
Let $\hat{f}_\theta$ be any globally Lipschitz function with constant $L_\theta$
(e.g.\ a bounded-parameter neural network). Then
\[
\lim_{\|x\|\to\infty} |\hat{f}_\theta(x)-f(x)| = \infty.
\]
\end{theorem}

\begin{proof}
By Lipschitzness, $|\hat{f}_\theta(x)| \le L_\theta\|x\| + |\hat{f}_\theta(0)|$
grows at most linearly, while $|f(x)|$ grows at least like $\|x\|^k$ with $k>1$.
Therefore the difference diverges as $\|x\|\to\infty$.
\end{proof}

\begin{remark}
This theorem does \emph{not} claim that all neural networks are globally
Lipschitz under all parameterizations; rather, it formalizes a common practical
constraint: deployed models are typically regularized/parameter-bounded, and
their effective growth is controlled. The conclusion is that exact symbolic
execution is the correct mechanism for unbounded mathematical laws when
extrapolation correctness is required.
\end{remark}

\begin{definition}[Neuro-symbolic grammar]
\label{def:neurosymbolic_grammar}
A \emph{neuro-symbolic grammar} is a pair $(\mathcal{R}_\phi, \mathcal{G}_{\mathrm{sym}})$
where:
\begin{itemize}
\item $\mathcal{R}_\phi: \mathcal{P}_2(\mathbb{R}^d) \to \{1, \ldots, G\}$ is a neural
      router (e.g.\ Wasserstein attention followed by hard selection),
\item $\mathcal{G}_{\mathrm{sym}} = \{g_1, \ldots, g_G\}$ is a library of symbolic grammars
      (exact computable functions).
\end{itemize}
The predicted execution is $\hat{e}_t = g_{\mathcal{R}_\phi(I_t)}(o_t)$.
\end{definition}

\begin{theorem}[Perfect generalization given correct routing]
\label{thm:perfect_generalization}
If the router $\mathcal{R}_\phi$ correctly identifies the active grammar $g^*$
for an instance, then the executed prediction is exact:
\[
\hat{e}_t = g^*(o_t) = e_t.
\]
\end{theorem}

\begin{proof}
If $\mathcal{R}_\phi(I_t)=g^*$, then $\hat{e}_t=g^*(o_t)$. Since the environment
execution is $e_t=g^*(o_t)$ under the active grammar, the equality holds.
\end{proof}

\subsection{The Agency Trajectory: From Observer to Planner}
\label{sec:agency_trajectory}

The preceding sections established BMA as a passive inference system: given
observations, it infers latent states and learns grammars. But intelligence
requires \emph{agency}---the capacity to act on the world, not merely observe it.
We now formalize a developmental trajectory from passive observation to active
planning, showing how the circular architecture $S_t \leftarrow H_{t-1}$ provides
the foundation for increasingly sophisticated agency.

\begin{definition}[Agency trajectory]
\label{def:agency_trajectory}
The \emph{agency trajectory} is a sequence of increasingly capable architectures:
\[
\text{Observer} \to \text{Agent} \to \text{Composer} \to \text{Planner},
\]
where each stage subsumes the previous and adds new capabilities.
\end{definition}

\paragraph{Stage 1: Observer.}
The system observes world transitions $(o_t, o_{t+1})$ without agency. It learns
to recognize which grammar generated the transition.

\paragraph{Stage 2: Agent.}
The system gains an action head that maps intent to action selection and closes
the causal loop
\[
S_t \to I_t \to a_t \to o_{t+1} \to H_t \to S_{t+1}.
\]

\paragraph{Stage 3: Composer.}
The system learns to compose multiple actions to achieve goals unreachable in
one step, chunking frequent compositions into higher-level grammars.

\paragraph{Stage 4: Planner.}
The system adds a world model for lookahead and replanning under uncertainty,
and integrates predictions into the hybrid latent:
\[
H_t = g(M_t, S_t, \hat{o}_{t+1}).
\]

\section{Discussion and Conclusion}

We introduced a mathematically explicit formulation of the Bayesian Mirror
Architecture as a circular latent generative system on belief space. The core
architectural condition is the active recursion $S_t \leftarrow H_{t-1}$,
enforcing self-referential closure independently of any optimization
interpretation.

We established that variational free energy minimization is the \emph{sole}
learning objective. The system learns by reducing its surprise about
observation--execution pairs; no auxiliary consciousness-targeting losses are
imposed. The learning metrics (Stability, Coherence, Integration) are
diagnostic measurements that indicate whether the learned model has captured
stable causal structure.

We treated learning as a Wasserstein gradient flow on belief space. In the
simplest limit the learning dynamics reduce to heat-type diffusion; in general
they realize drift--diffusion towards posterior consistency.

Crucially, we did \emph{not} assume global strict contractivity. Instead, we
introduced basin-wise dynamics and proved existence and uniqueness of a
basin-restricted invariant measure under local Wasserstein contractivity,
defining the self-manifold as the support of that basin-wise invariant measure.

We introduced Wasserstein $\epsneck$-necks as transport bottlenecks separating
coherent belief basins and interpreted realized basin continuation under
vanishing conductance as \emph{choice}. Finally, we formulated intent as a
Wasserstein attention query over causal grammar prototypes and motivated a
neuro-symbolic separation between routing and execution for exact mathematical
laws.

\end{document}